\documentclass{ecai} 
\usepackage{fix-cm}

\usepackage{latexsym}
\usepackage{amssymb}
\usepackage{amsmath}
\usepackage{amsthm}
\usepackage{booktabs}
\usepackage{enumitem}
\usepackage{graphicx}
\usepackage{color}
\usepackage[dvipsnames]{xcolor}
\usepackage{algpseudocode}
\usepackage{algorithm}
\usepackage{float}

\newtheorem{theorem}{Theorem}
\newtheorem{definition}{Definition}
\newtheorem{lemma}[theorem]{Lemma}

\newcommand{\BibTeX}{B\kern-.05em{\sc i\kern-.025em b}\kern-.08em\TeX}

\begin{document}
\nolinenumbers


\begin{frontmatter}


\paperid{1197} 


\title{PLACO: A Multi-Stage Framework for Cost-Effective Performance in Human-AI Teams}


\author[A]{\fnms{Pranavkumar}~\snm{Mallela}}
\author[A]{\fnms{Vinay}~\snm{Kumar}}
\author[A]{\fnms{Shashi}~\snm{Shekhar Jha}\thanks{Corresponding Author. Email: shashi@iitrpr.ac.in.}}
\author[A]{\fnms{Shweta}~\snm{Jain}\thanks{Corresponding Author. Email: shwetajain@iitrpr.ac.in.}}

\address[A]{Indian Institute of Technology Ropar}


\begin{abstract}
Human-AI teams have a pervasive impact in various fields including healthcare diagnosis, robotics in manufacturing, cyber-security, autonomous vehicles, and many more. The effectiveness of Human-AI teams highly depends on the set of humans that interact with the AI model for determining the final output. In this paper, we tackle the practical setting where taking the human input is of considerable cost and even expert humans can make mistakes. This paper proposes Probabilistic Labeler Assisted Cost Optimization (PLACO), a two-step framework to find cost-effective subsets of humans for multi-way classification tasks. The inputs from the subset of humans are then combined with the AI model’s output resulting in the most accurate output. For cost-effective human selection given an input task, we estimate human labels by maximizing the posterior probability of a true human label given the AI model’s output on the task. We further derive a value function that determines the value of a given human subset to maximize the lower bound on the overall accuracy of the Human-AI team. We present the theoretical foundations of our human label estimation method and human subset value function. We also empirically demonstrate the effectiveness of PLACO in terms of the Human-AI team’s performance and cost-effectiveness against state-of-art methods on the CIFAR-10H and Imagenet-16H datasets having human annotations.

\end{abstract}

\end{frontmatter}


\vspace{-1.25mm}
\section{Introduction}

Human-AI teams play a pivotal role in improving overall system performance when neither the human nor the model can achieve such performance on their own. With the advent of powerful and accessible Generative AI models, several mundane tasks have morphed into Human-AI team tasks. From writing essays to developing advanced algorithms, humans have found that using AI assistance has led to an accelerated work pace like never before.

In classification tasks, where the final output is a single hard label, it is crucial to address the combination of human and model output. Prior work \cite{kerrigan2021combining} elegantly solves this problem using Bayes’ rule, using the assumption that human and model output are conditionally independent given the ground truth. Specifically, it discusses a combination method to combine a single deterministic labeler (the human) and a probabilistic labeler (the classifier model) using the model’s instance-level and the human’s class-level calibrated probabilities.
However, \citet{kerrigan2021combining} considered combining human input with the AI model's probabilistic output in a Human-AI team having only a single human along with the AI model. The follow-up work by \citet{singh2023subset} extended the Human-AI team formation with multiple humans along with the AI model. The authors showed that the overall team with multiple humans resulted in a better performance because a single human may not have the required domain knowledge or expertise to always maximize the performance of the overall system. Thus, it is beneficial to combine the inputs from multiple humans while forming a Human-AI team. 
In addition, the authors in \cite{singh2023subset} propose a greedy algorithm to select an optimal subset of humans with the assumption that human labels across all tasks can be readily acquired. This particular assumption that the humans labels are available for all tasks inherently considers that human input is of negligible cost. It further limits this Human-AI team approach for any practical applications.

In the real world, getting the labels from humans may be costly. For example, a hospital that relies on AI input to make complex diagnoses will need to consult a doctor before a decision is taken. It is desirable both from the doctor's and the hospital's perspective to use human input only as required. This saves the doctor's time and hospital resources. Thus, although Human-AI combinations considerably enhance performance when compared to the model or the human alone, for the decisions to be made in the real world, the cost of human resources must be considered.

In this paper, we propose Probabilistic Labeller-Assisted Cost Optimization (PLACO), a framework for Human-AI collaboration. Our proposed framework selects cost-optimized subset of humans having varying cost. Further, our approach relaxes the assumption that labels of all humans are required on all instances. Our goal is to select a subset of humans forming a Human-AI team that is both accurate and cost-effective. In order to achieve this, our proposed framework follows a two-step approach. In the first step, given a classification instance, we estimate human labels using human confusion matrices and instance-level AI model's confidence. We show that it significantly outperforms na\"ive estimation methods. In the second step, we use a custom value function that utilizes the aforementioned estimated human labels to determine the value of human subsets. This presents the need for an algorithm to choose the subset of humans with the highest value, while minimizing the cost. PLACO allows flexibility in the use of subset selection algorithms as a "plug and play" component of the framework.
In this paper, we adapt the greedy algorithm presented in \cite{singh2023subset} as PLACO Greedy Subset Selection. Once the algorithm optimizes the subset with respect to the value and cost function, we elicit the true human labels only from the selected subset of humans, thereby significantly bringing down the cost of the end-to-end process. 
We show that PLACO delivers comparable, if not better accuracy than existing state-of-the-art. We use four different human configurations with varying accuracies while making the entire process significantly more cost-effective.

\vspace{-1.25mm}
\section{Related Work in Human-AI Teams}
Existing literature offers several setups in which
human experts augment and complement AI/ML models. Human-AI Teams are employed in supervised learning, semi-supervised learning, unsupervised learning and reinforcement learning (e.g. \cite{wu2023toward}). Broadly, two main setups emerge from the literature: one involves human effort during data preprocessing or model training, while the other incorporates human input post-model training. While the literature \cite{wu2022survey, mosqueira2023human} extensively explores the integration of human expertise with machine learning models, our approach specifically aligns with the second type of setting.

\citet{hendrycks2016baseline} propose a method to utilize human input only in instances of low model confidence. However, this strategy encounters limitations, as human involvement may not consistently improve overall performance in such scenarios. \citet{madras2018predict} introduce an alternative approach known as the defer model, where the model learns to defer to human judgment when necessary. Many researchers have pursued this approach as is evident in the works \cite{mozannar2020consistent,verma2022calibrated,Gao2021}. \citet{gupta2023take} delve into the cost of expert input and misclassification cost of an instance, based on which the model decides whether or not to defer to a human classifier. In another work by \citet{keswani2021towards}, the model defers a low-confidence instance to a set of humans based on their respective expertise. However, it is important to acknowledge the drawbacks of the learning-to-defer approach, as discussed by \citet{leitao2022human}. This approach tends to specialize in instances of high confidence and may be susceptible to changes in data distribution. Further to accurately learn which instances should be deferred to the human expert, a large number of expert predictions that accurately refect the expert’s capabilities are required. \citet{hemmer2023learning} propose a three-step approach to reduce the number of expert predictions required to train learning to defer algorithms. \citet{babbar2022utility} proposes a unique framework, D-CP, which combines the defer model and conformal prediction. Through human subject experiments, they show increased levels of trust and utility while using conformal prediction instead of the Top-1 model prediction.

Another Human-AI approach considered in literature is that of the AI-assisted setting. Works by \citet{bansal21} have allowed humans to decide whether to incorporate the AI model's output or to solve the instance independently. \citet{fuchs2024optimizing} propose the introduction of a manager which learns how to delegate tasks to human and AI agents through a Reinforcement Learning scheme across different grid environments and risk-aversion levels. The nature of Human-AI interaction also plays a significant role in such settings as demonstrated by \citet{bondi2022role}. The work \cite{tutulhuman} investigates the levels of human trust in the AI model in the detection of deceptive speech. \citet{martinez2021improving} delve into a personalized approach of the AI-assisted setting in which they consider a set of humans and define personalized loss functions to increase Human-AI team accuracy. Another conceptual model, the Shared Memory Model (SMM), was described in detail in the work \cite{andrews2023role}. The authors discuss the factors that influence the formation of an effective SMM between a human and an AI agent including ability to acquire each other's mental models, trust, and the human-machine interface among others. 

In addition to the AI-assisted setting, authors \citet{kerrigan2021combining} have introduced and worked on the combination approach to improve overall system accuracy. Specifically, the work \cite{kerrigan2021combining} aims to generate a single label from a human’s non-probabilistic output and an independently trained model’s probabilistic output. Further enhancing this idea, \citet{singh2023subset} combine the non-probabilistic output from a subset of humans. They also show that the increase in accuracy is non-monotone, thus accentuating the need for an effective subset selection strategy.

\citet{hemmer2022forming} propose the training of a classifier to determine which instances are difficult for humans, while also training an allocation system to assign a given instance to the most suitable team member. The authors train the classifier to  consider the individual capabilities of human experts thereby leveraging the advantages of a diverse human set. Recent work \cite{zhang2023learning} presents Learning to Complement with Multiple Humans (LECOMH), an innovative approach that integrates the usage of noisy-label learning, multi-rater learning and human-AI collaboration. The authors aim to optimize human-AI collaboration by maximizing accuracy while minimizing collaboration cost.  Authors \citet{tariq2024a2c} propose the framework A2C, which facilitates three decision-making models: Automated, Augmented, and Collaborative, with emphasis on cyber-security. 

In the light of such frameworks and methodologies, we see that the inclusion of the cost of human input presents a unique challenge that must be addressed for real-world deployment. In our approach, we extend the Human-AI combined model approach for a multi-way classification task by finding the subset of humans with low cost and high value.

\vspace{-1.25mm}
\section{Preliminaries}

\begin{figure}[t]
    \centering
    \includegraphics[width=1\linewidth]{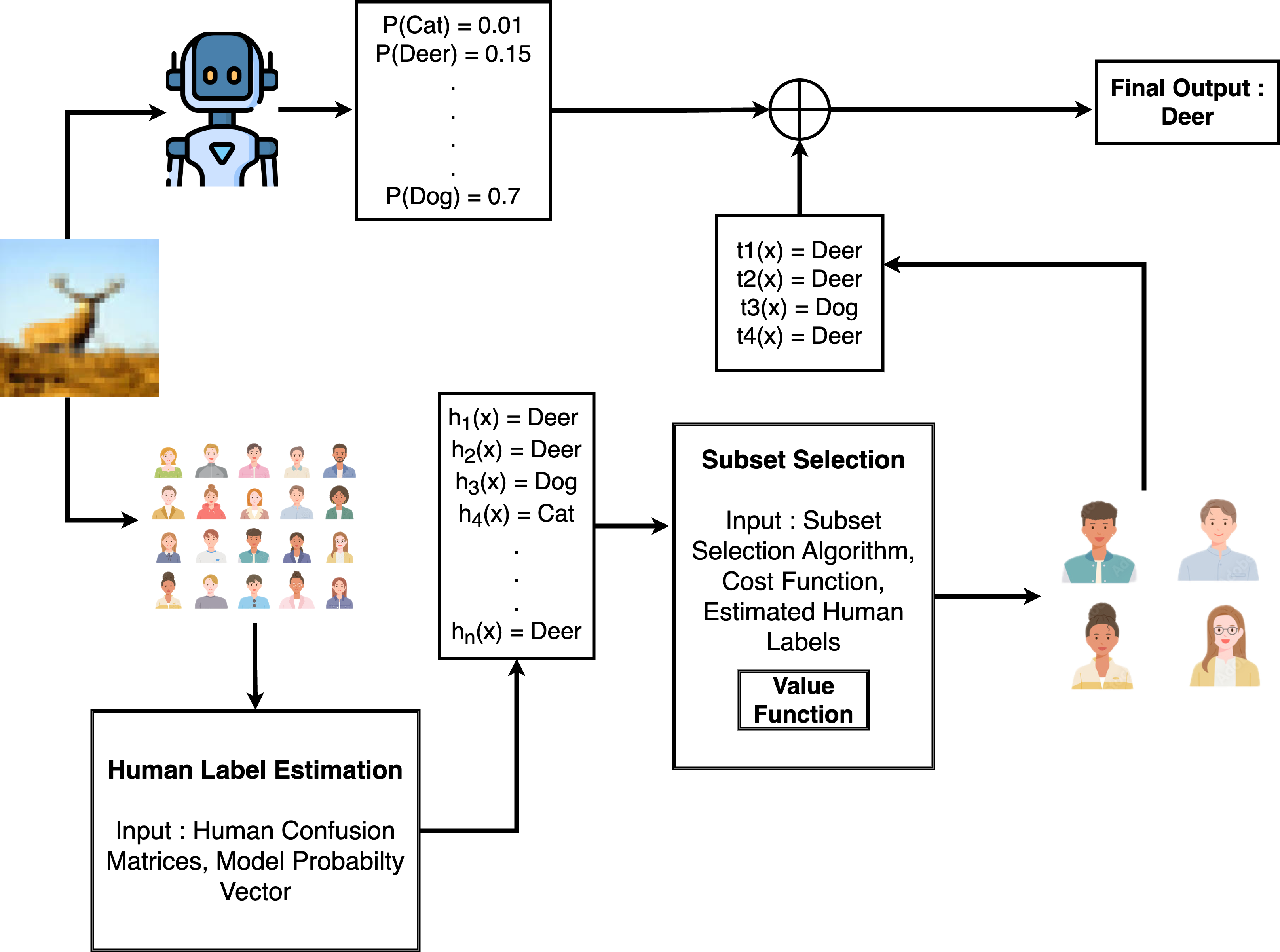}
    \vspace{-5mm}
    \caption{Illustration of Probabilistic Labeller-Assisted Cost Optimization (PLACO) Framework for Human-AI Collaboration}
    \label{fig:placo}
\end{figure}

Given a set of humans $\mathcal{H}$ and a model $m$ our aim is to select subset $\mathcal{S}_x \subset \mathcal{H}$,  for each instance $x \in \mathcal{X}$ so as to maximize the accuracy on a $k$-way classification task, while significantly reducing the cost of human input. For multi-way classification task, we consider the label set $\mathcal{Y} = \{0, 1, 2, ..., k-1\}$. Let $y(x) \in \mathcal{Y}$ denote the ground truth of an instance $x \in \mathcal{X}$. Let $m(x)$ be the AI model’s $k$-dimensional probability vector and $t_i(x)$ be the human label of the $i^{th}$ human. Consider $t(\mathcal{S}_x)$ be the collection of labels from humans in the subset  $\mathcal{S}_x$, that is, $\{t_1(x), t_2(x), ..., t_{|\mathcal{S}_x|}(x)\}$.

For a significant reduction in the cost of the required human input, we perform an estimation of human labels, in contrast to previous approaches \cite{singh2023subset} which utilize true human labels on all instances $x \in \mathcal{X}$ to select the subsets $S_x$. We denote the estimated label of the $i^{th}$ human by $h_i(x) \in \mathcal{Y}$.  Our estimation of the label of the $i^{th}$ human on a given instance $x \in \mathcal{X}$ mainly depends on two quantities: the confusion matrix of the $i^{th}$ human and the AI model probability vector $m(x)$.
We define the human confusion matrix as  ${\phi^{[i]}_{st}} = P(t_i(x) = s | y(x) = t)$ and use a Bayesian combination approach to combine $m(x)$ and $t(\mathcal{S}_x)$ to obtain the final prediction $c(x)$.

Figure \ref{fig:placo} illustrates the step-by-step procedure we employ to combine model predictions with human labels of the cost-effective human subset to derive the final classification. The classification model incorrectly predicts the label \lq Dog\rq with a probability of 0.70. Using the model's probability vector $m(x)$ and the human confusion matrices $\phi$, the framework first estimates the labels for the human population. The cost-effective subset is then chosen using these estimated labels, the user-provided Subset Selection Algorithm, the user-provided Cost Function, and the framework-integrated Value Function. The combination of labels of the cost-effective subset and model's probability vector predicts the correct final label \lq Deer\rq.

\subsection{Human Confusion Matrices}
As mentioned earlier, the confusion matrix of the $i^{th}$ human is essential to determine $h_i(x)$ for a given instance $x$. However, such confusion matrices are not readily available, thus have been estimated in previous approaches \cite{singh2023subset,kerrigan2021combining} using true human labels $t(\mathcal{H}) = \{t_1(x), t_2(x), ..., t_{|\mathcal{H}|}(x)\}$. We also adopt this approach to estimate human confusion matrices. We reserve some instances for training human confusion matrices from the human-labelled data. However, we do not require true human labels on future instances.
Existing literature \cite{kerrigan2021combining, caelen2017bayesian, venanzi2014community, singh2023subset} demonstrates that employing a Dirichlet prior for each column of the confusion matrix yields an efficient estimation. This approach requires fewer human labels to estimate the confusion matrix as it incorporates additional information through the use of the Dirichlet prior.

\begin{equation*}
    \phi^{[i]}_{st} = Dirichlet(\alpha_{t})_{s}, \hspace{0.25cm} \forall t
\end{equation*}

Given $\mathcal{K}$ classes, the prior parameter $\alpha_{t} \in \mathbb{R}^{\mathcal{K}}$ is chosen such that
\begin{equation*}
(\alpha_{t})_{k} = 
\begin{cases}
\beta,\;  k \neq t\\
\gamma,\;  k = t    
\end{cases}
\end{equation*}
where $\beta, \gamma \in \mathbb{R}_{+}$.\\
The resulting prior matrix exhibits $\gamma$ along the diagonal and $\beta$ on the off-diagonal. We opt for a Dirichlet prior for each column, ensuring that all off-diagonal prior values are identical.

Before we delve into the combination method, we must briefly describe the two essential steps of our PLACO framework. First, we utilize the above-mentioned human confusion matrices $\phi$ and model probabilities $m(x)$ on a given instance $x$ to determine the estimated human labels $h(x)$. The method of estimation is described in detail in Section \ref{subsection:label-estimation}. Second, a given subset selection algorithm uses the estimated human labels $h(x)$ in our proposed value function to select the highest valued subset $\mathcal{S}_x$ of humans. At this point, instead of all the true human labels, only the true human labels $t(\mathcal{S}_x)$ are elicited and thus, have some cost. As we are equipped with the necessary true human labels on a given instance $x$, we can now describe the combination method.

\subsection{Bayes' Combination of Model and Human Output}
Existing literature describes a dichotomy of methods to arrive at a final prediction of a Human-AI team. The first set of methods follow the deferral approach \cite{keswani2021towards,leitao2022human,madras2018predict,mozannar2020consistent}, that is, the model learns when to defer to an expert human. The expert human typically requires significant cost. However, the expert provides the correct label with high probability. In our practical setting where humans can make mistakes and whose input is of considerable cost, we posit that the deferral approach is not sufficient. Thus, we look toward the second set of methods that follow the combination approach, that is, human input and model probabilities are combined to yield a final prediction. The literature on crowdsourcing \cite{jain2018quality,lamberson2012optimal} focuses on combining multiple human labels and considers that as the predicted label to combine with the model. \citet{kerrigan2021combining} presents a Bayesian approach to combine a single human label with the model which is given as:
\begin{equation}
    P(y(x) = j | t(x) = i, m(x)) = 
    \frac{m_{j}(x)\phi_{ij}}
    {\sum_{k = 1}^{\mathcal{K}} m_{k}(x)\phi_{ij}}
\end{equation}
As we utilize multiple humans in our Human-AI teams, the combination method we require must consider the inputs from a set of humans. Inspired from \citet{kerrigan2021combining}, \citet{singh2023subset} present a probabilistic approach \textit{ComHAI} to combine \textit{n} human labels with the model. This is the combination method we employ in our framework given as:
\begin{equation}
    P(y(x) = j | t(\mathcal{H}), m(x)) = 
    \frac{m_{j}(x)\prod_{i\in [n]}\phi^{[i]}_{t_i(x)j}}
    {\sum_{k = 1}^{\mathcal{K}} m_{k}(x)\prod_{i\in [n]}\phi^{[i]}_{t_i(x)k}}
\end{equation}

where, $t(\mathcal{H}) = \{t_1(x), t_2(x), ..., t_{|\mathcal{H}|}(x)\}$.

\begin{figure*}[t]
    \centering
    \includegraphics[width=0.245\linewidth]{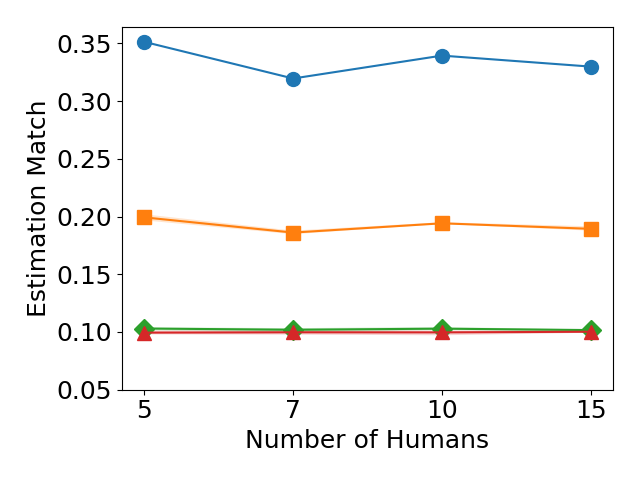}
    \includegraphics[width=0.245\linewidth]{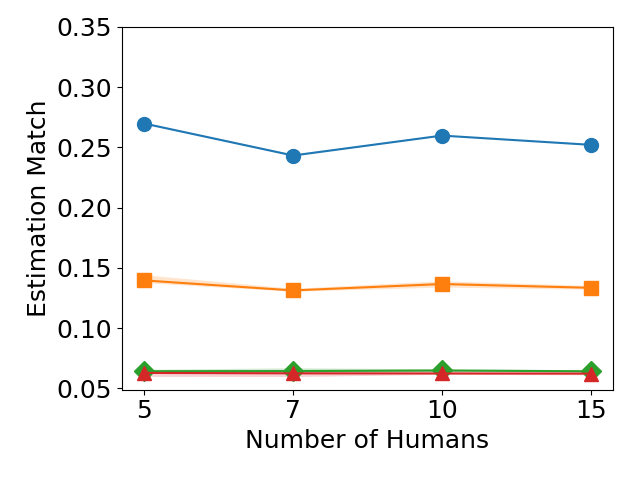}
    \raisebox{0.8\height}{\includegraphics[height=0.1\linewidth]{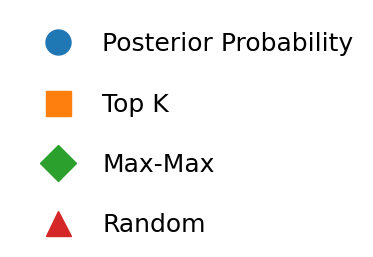}}
    \vspace{-3mm}
    \caption{Estimation Match Comparison on CIFAR-10H and ImageNet-16H respectively across different human configurations. Estimation match is the average fraction of correctly estimated human labels on a given instance. A CNN model with accuracy 56\% is used as AI model for probabilistic output for CIFAR-10H and a another CNN model with accuracy 43\% is used for ImageNet-16H to combine with human labels.}
    \label{fig:estimation_match}
\end{figure*}

\vspace{-1.25mm}
\section{Proposed Approach: PLACO}
In this section, we describe the PLACO framework to produce cost-effective and highly performant Human-AI teams. We begin with proposing our method to estimate human labels using posterior probabilities $P(t_i(x)|m(x))$. Then, we describe our value function which we employ in a given subset selection algorithm to maximize accuracy while considering the cost of human input. Figure \ref{fig:placo} contains the overall flow to produce the final combined prediction.

\subsection{Human Label Estimation using Posterior Probability}\label{subsection:label-estimation}
In order to produce a significant boost in accuracy, we choose the human label $h_i(x)$ that maximizes $P(t_i(x)|m(x))$, where $m(x)$ is the model's probabilistic output. The probability of a particular label $l$ is chosen by the $i^{th}$ human given the model's output $m(x)$ can be written as
    \begin{align*}
        P(t_i(x) = l|m(x)) & = \sum_{y \in \mathcal{Y}} P(t_i(x) = l|m(x),y) \cdot P(y|m(x))
    \end{align*}

With the assumption that the human and model are independent of each other given the ground truth $y$, the above equation leads to,

\begin{align*}
    P(t_i(x) = l|m(x)) & = \sum_{y \in \mathcal{Y}} P(t_i(x) = l|y) \cdot P(y|m(x)) 
\end{align*}

The term $P(t_i(x) = l|y)$ can be obtained from the human confusion matrix, ${\phi^{[i]}_{ly}}$. The term $P(y|m(x))$ is the model probability for a given label $y$. We determine the estimated label for the $i^{th}$ human as the label that maximizes this posterior probability. That is, 

\begin{equation}
    {h_i(x) = \arg\max_{l \in \mathcal{Y}} \sum_{y \in \mathcal{Y}} P(t_i(x) = l| y) \cdot P(y | m(x))}
\end{equation}

We present the estimation effectiveness with respect to other methods in Section \ref{section:experiments} and \ref{section:results}.

\subsection{Value Function for Cost-Effective Human Subset Selection} \label{subsection:value-function}
As mentioned earlier, it is essential to select the most suitable set of humans to assist the AI model in order to produce the most effective final prediction. 
Our custom value function forms the core of the subset selection process. Recent work \cite{singh2023subset} proposes a greedy subset selection algorithm that significantly outperforms na\"ive subset selection methods. However, this work assumes zero-cost human labels. 
We employ our own PLACO Greedy Subset Selection and an off-the-shelf linear solver to perform cost-effective subset selection, which are elaborated in Section \ref{subsection:subset-selection}.

The value function should reward humans who increase overall system's accuracy. As we cannot know the overall accuracy improvement at the time of subset selection, we invoke the lower bound on the accuracy of the combined model proven in \cite{singh2023subset}. It gives the lower bound as,
\begin{equation}\label{eq-truelbsubset}
\fontsize{8}{9}\selectfont
    \mathbb{E}[\mathbf{1}(c(x) = y(x))] \geq \mathbb{P} \left\{
    \prod_{i\in [n]} \frac{\phi^{[i]}_{t_{i}(x)y(x)}}{1 - \phi^{[i]}_{t_{i}(x)y(x)}}
    > \frac{1-m_{y(x)}(x)}{m_{y(x)}(x)}
    \right\}
\end{equation}

Here, as done in \cite{singh2023subset}, one can maximize the term $\left(\prod_{i\in [n]} \frac{\phi^{[i]}_{t_{i}(x)y(x)}}{1 - \phi^{[i]}_{t_{i}(x)y(x)}}\right)$ to maximize the accuracy lower bound. However, we tackle the setting in which human labels $t(x)$ are unavailable. As a result, we aim to find a lower bound of the ratio $\frac{\phi^{[i]}_{t_{i}(x)y(x)}}{1 - \phi^{[i]}_{t_{i}(x)y(x)}}$ using estimated human labels $h(x)$. If we can maximize the lower bound, we will have come as close as possible to maximizing the original ratio. We define the following term in order to achieve the same.

\begin{definition} (Ideal Human)
 We define the $i^{th}$ human on a given instance $x$ to be ideal if, $t_i(x) = y(x)$.
\end{definition}

\begin{lemma}\label{range} On a given instance $x$, for an ideal human $i$
    \begin{equation}\label{}
   2a_i - 1 \leq \phi^{[i]}_{t_{i}(x)y(x)} - \phi^{[i]}_{h_{i}(x)y(x)} \leq A_i
    \end{equation}
and for a non-ideal human $i$
    \begin{equation}\label{}
   -1 \leq \phi^{[i]}_{t_{i}(x)y(x)} - \phi^{[i]}_{h_{i}(x)y(x)} \leq 1
    \end{equation}
    where, $a_i = \min_{l \in \mathcal{Y}} \phi^{[i]}_{ll}$ \& $A_i = \max_{l \in \mathcal{Y}} \phi^{[i]}_{ll}$
\end{lemma}

\begin{proof} On a given instance $x$,\\
For an ideal human $i$, $t_i(x) = y(x)$
\begin{align*}
    &\phi^{[i]}_{t_{i}(x)y(x)} = \phi^{[i]}_{y(x)y(x)} \leq \max_{l \in \mathcal{Y}} \phi^{[i]}_{ll} = A_i \; \& \; \phi^{[i]}_{h_{i}(x)y(x)} \geq 0
\end{align*}
\begin{align*}
    &\implies \phi^{[i]}_{t_{i}(x)y(x)} - \phi^{[i]}_{h_{i}(x)y(x)} \leq A_i
\end{align*}
Now,
\begin{align*}
    &\phi^{[i]}_{t_{i}(x)y(x)} = \phi^{[i]}_{y(x)y(x)} \geq \min_{l \in \mathcal{Y}} \phi^{[i]}_{ll} = a_i \\
    &\phi^{[i]}_{h_{i}(x)y(x)} \leq 1 - \phi^{[i]}_{t_{i}(x)y(x)} \leq 1 - a_i
\end{align*}
\begin{align*}
    &\implies \phi^{[i]}_{t_{i}(x)y(x)} - \phi^{[i]}_{h_{i}(x)y(x)} \geq a_i - (1 - a_i) = 2a_i - 1
\end{align*}
For a non-ideal human $i$, 
\begin{align*}
    &0 \leq \phi^{[i]}_{t_{i}(x)y(x)} \leq 1 \; \& \; 0 \leq \phi^{[i]}_{h_{i}(x)y(x)} \leq 1
\end{align*}
\begin{align*}
    \implies -1 \leq \phi^{[i]}_{t_{i}(x)y(x)} - \phi^{[i]}_{h_{i}(x)y(x)} \leq 1&&\qedhere
\end{align*}
\end{proof}

 Now, armed with the range of values that $\left(\phi^{[i]}_{t_{i}(x)y(x)} - \phi^{[i]}_{h_{i}(x)y(x)}\right)$ can take for an ideal human, we proceed to derive the value function.

 \begin{lemma}\label{ratio-lower-bound}
     For an ideal human $i$, on a given instance $x$, the following relation exists:
     \begin{equation}\label{lowerbound_ideal}
   \frac{\phi^{[i]}_{t_{i}(x)y(x)}}{1 - \phi^{[i]}_{t_{i}(x)y(x)}} \geq \frac{\phi^{[i]}_{h_{i}(x)y(x)} + 2a_i - 1}{2 - (\phi^{[i]}_{h_{i}(x)y(x)} + 2a_i)}
    \end{equation}
    while, for a non-ideal human $i$, we have:
    \begin{equation}\label{lowerbound_nonideal}
   \frac{\phi^{[i]}_{t_{i}(x)y(x)}}{1 - \phi^{[i]}_{t_{i}(x)y(x)}} \geq \frac{\phi^{[i]}_{h_{i}(x)y(x)} - 1}{2 - \phi^{[i]}_{h_{i}(x)y(x)}}
    \end{equation}
 \end{lemma}

 \begin{proof}
     The above relations come directly from substituting the lower bound derived in Lemma \ref{range} in the ratio $\frac{\phi^{[i]}_{t_{i}(x)y(x)}}{1 - \phi^{[i]}_{t_{i}(x)y(x)}}$ .\\ That is, for an ideal human $i$, 
     \begin{align*} 
        \phi^{[i]}_{t_{i}(x)y(x)} \geq \phi^{[i]}_{h_{i}(x)y(x)} + 2a_i - 1
     \end{align*}
     For a non-ideal human $i$,
     \begin{align*}
        \phi^{[i]}_{t_{i}(x)y(x)} \geq \phi^{[i]}_{h_{i}(x)y(x)} - 1&&\qedhere
     \end{align*}
 \end{proof}
\bigskip
 We utilize these relations accordingly in order to present our proposed value function for ideal and non-ideal humans. However, as we cannot know whether a human is ideal or not on a given instance $x$ due to the lack of information of $t_i(x)$ and $y(x)$, we resort to the probability $P(t_i(x) = y(x))$ i.e. accuracy of the $i^{th}$ human as an indicator. Also, to compensate for the unavailability of $y(x)$, we use a label $j$ in its place which is taken as a parameter by the value function. Equation \ref{y_star_eq} shows how we select the label $j$ to give us an approximation of $y(x)$.
 
 Additionally, in order to proceed, it is essential to handle the corner cases in Lemma \ref{ratio-lower-bound}. For an ideal human, as the sum $\left(\phi^{[i]}_{h_{i}(x)y(x)} + 2a_i\right)$ approaches the value 2, the lower bound in Equation \ref{lowerbound_ideal} reaches arbitrarily high values. Hence, if the sum exceeds 2, we assign a very high value $V_{max}$ to that human. Similarly, as $\left(\phi^{[i]}_{h_{i}(x)y(x)} + 2a_i\right)$ approaches the value 1, the ratio approaches 0. In order to avoid misleading comparisons, we assign a very small positive value $\epsilon$ instead of 0. In the case of a non-ideal human, Lemma \ref{ratio-lower-bound} gives a non-positive lower bound. Similar to the previous case, we assign a very small positive value $\epsilon$ to avoid misleading comparisons. $V_{max}$ and $\epsilon$ are specified in Section \ref{subsection:subset-selection}.

Incorporating the above cases, our PLACO Framework proposes the following piecewise function as the value function for the $i^{th}$ human on given instance $x$ for a label $j$:
\begin{equation}\label{value_function}
V_i(x, j) = 
\begin{cases}
    \frac{\phi^{[i]}_{h_{i}(x)j} + 2a_i - 1}{2 - (\phi^{[i]}_{h_{i}(x)j} + 2a_i)} & \text{if } \begin{aligned}[t]
                        &P(t_i(x) = y(x)) \geq 0.5 \text{ and }\\
                        &\phi_{h_i(x)j} + 2a_i \in (1,2)
                     \end{aligned}\\\\
    \epsilon & \text{if } \begin{aligned}[t]
                        &P(t_i(x) = y(x)) \geq 0.5 \text{ and }\\
                        &\phi_{h_i(x)j} + 2a_i \leq 1
                     \end{aligned}\\\\
    V_{max} & \text{if } \begin{aligned}[t]
                        &P(t_i(x) = y(x)) \geq 0.5 \text{ and }\\
                        &\phi_{h_i(x)j} + 2a_i \geq 2
                     \end{aligned}\\\\
    \epsilon & \text{if } \begin{aligned}[t]
                        &P(t_i(x) = y(x)) < 0.5
                     \end{aligned}
\end{cases}
\end{equation}
where, $a_i = \min_{l \in \mathcal{Y}} \phi^{[i]}_{ll}$. 

We define the label $y^*$ which is an approximation of $y(x)$ as, 
\begin{equation}\label{y_star_eq}
    y^* = \arg\max_{j \in \mathcal{Y}}\prod_{i\in \mathcal{H}} V_i(x,j)
\end{equation}
It is natural to choose the label $y^*$ in this way as it maximizes the product of human values which in turn maximizes the probability in Equation \ref{eq-truelbsubset}. This is because the value function itself is derived as the lower bound of the ratio $\frac{\phi^{[i]}_{t_{i}(x)y(x)}}{1 - \phi^{[i]}_{t_{i}(x)y(x)}}$ in Equation \ref{eq-truelbsubset}.

\begin{figure*}[t]
    \centering
    \includegraphics[width=0.245\linewidth]{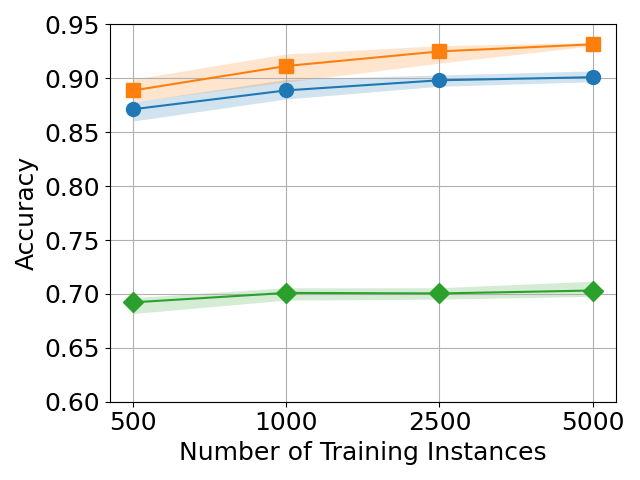}
    \includegraphics[width=0.245\linewidth]{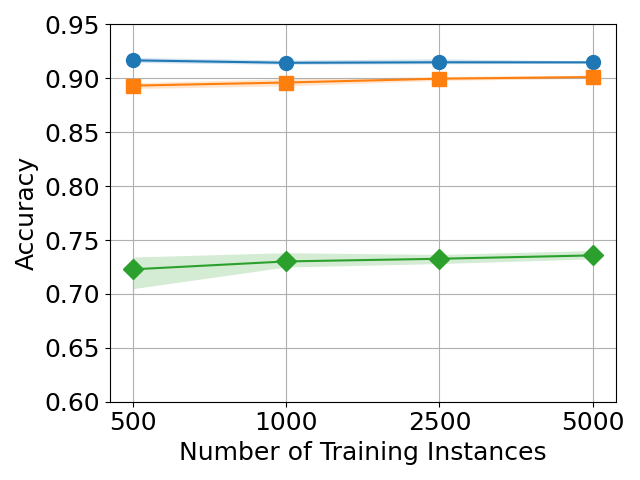}
    \includegraphics[width=0.245\linewidth]{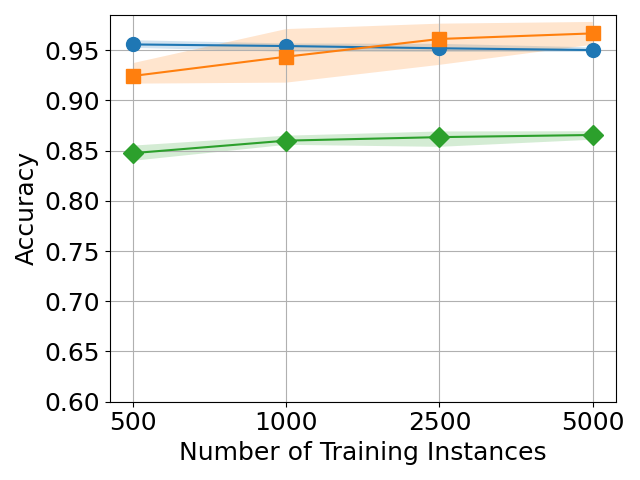}
    \includegraphics[width=0.245\linewidth]{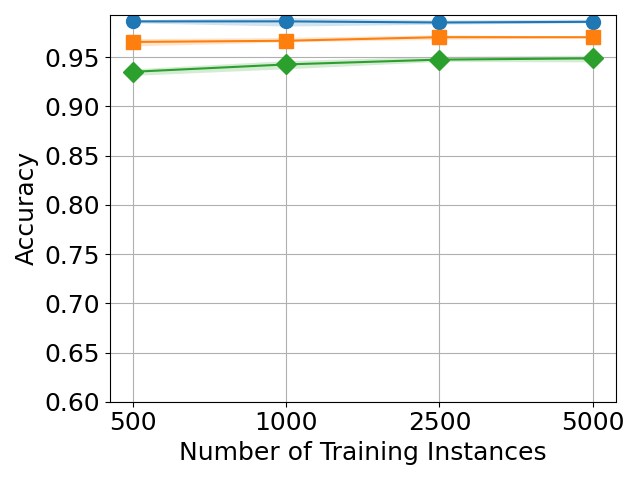}
    \includegraphics[width=0.245\linewidth]{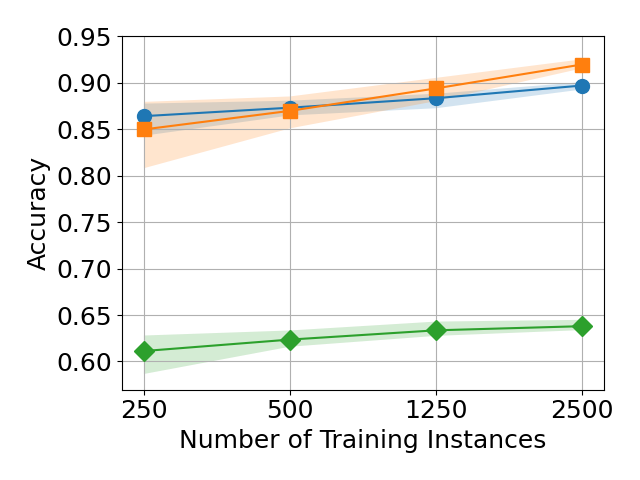}
    \includegraphics[width=0.245\linewidth]{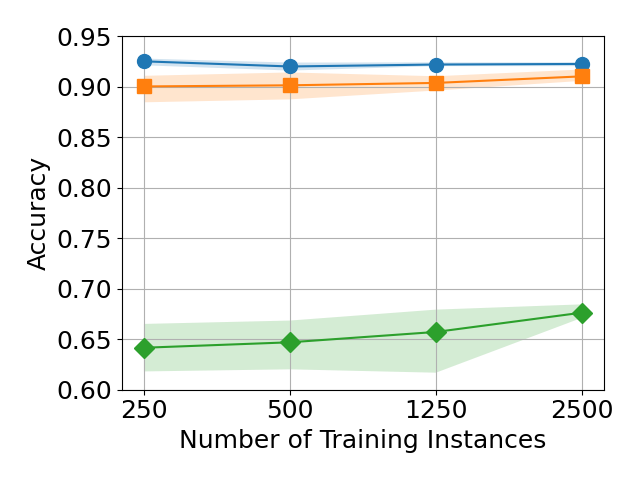}
    \includegraphics[width=0.245\linewidth]{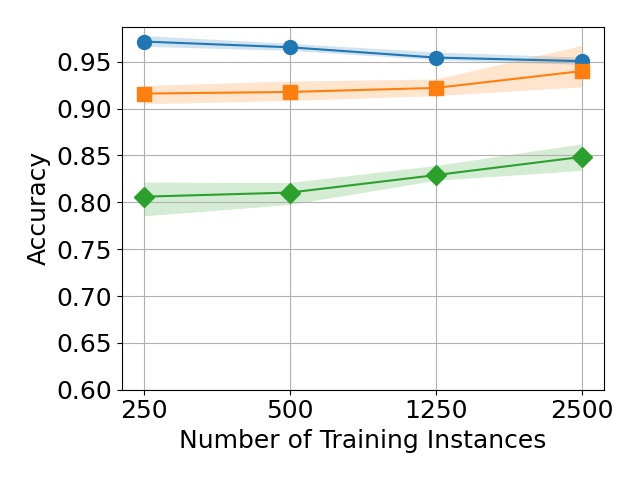}
    \includegraphics[width=0.245\linewidth]{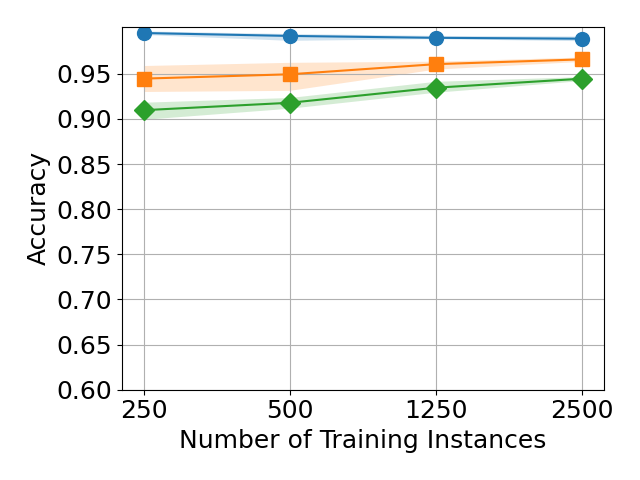}
    \includegraphics[width=0.4\linewidth, height=0.03\linewidth]{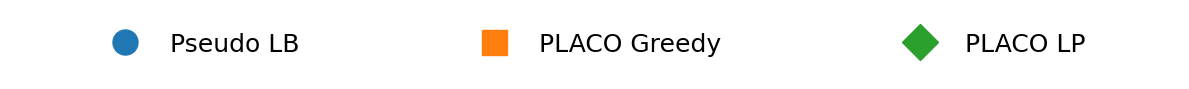}
    \vspace{-3mm}
    \caption{Learning curve with different subset selection methods average over 10 runs. Each plot corresponds to a different human configuration of varying accuracies. Specifically, the number of humans in the configurations are 5, 7, 10 and 15 (from left to right in each row), with accuracies ranging from 0.3 to 0.9. The first row corresponds to CIFAR-10H and the second corresponds to ImageNet-16H. A CNN model with accuracy 56\% is used as AI model for probabilistic output for CIFAR-10H and a another CNN model with accuracy 43\% is used for ImageNet-16H to combine with human labels.}
    \label{fig:accuracycifar}
\end{figure*}


\subsection{PLACO Greedy Subset Selection}
To maximize performance, we must select the subset $S_x$ for a given instance $x$ that maximizes the product of human values, that is,
\begin{align*}
    S_x = \arg\max_{S \subseteq \mathcal{H}}\prod_{i\in S} V_i(x, y^*)
\end{align*}
Now, we maximize this product by first choosing the human with the highest value. Then, we only add humans if they increase the product. Clearly, humans added in this phase will be chosen only if their value $V_i(x, y^*)$ is greater than 1. We present our PLACO Greedy algorithm in Algorithm \ref{placo_greedy_algo}.

\begin{algorithm}
\caption{PLACO Greedy Subset Selection}\label{placo_greedy_algo}
\begin{algorithmic}[1]
\Require Instance $x \in \mathcal{X}$, human set $\mathcal{H}$, cost function $C$
\State Let $subset \gets \{\}, subset\_cost \gets 0$
\State Let $max\_human \gets -1, max\_value \gets -1$
\State Let $y^* \gets \arg\max_{j \in \mathcal{Y}}\prod_{i\in \mathcal{H}} V_i(x,j)$
\For{$i \gets 1$ to $|\mathcal{H}|$}
    \State $human\_value \gets V_i(x, y^*)$
    \If{$human\_value > max\_value$}
        \State $max\_value \gets human\_value$
        \State $max\_human \gets i$
    \EndIf
    \If{$human\_value > 1$}
        \State $subset \gets subset \cup i$
        \State $subset\_cost \gets subset\_cost + C(i)$
    \EndIf
\EndFor
\If{$subset$ is empty}
    \State $subset \gets subset \cup max\_human$
    \State $subset\_cost \gets subset\_cost + C(max\_human)$
\EndIf
\State \textbf{Output} $subset$, $subset\_cost$
\end{algorithmic}
\end{algorithm}

Lines 1-2 initialize the subset, subset cost, and variables to store the highest valued human. Line 3 determines the label $y^*$. Lines 4-14 perform condition checks on every human in $\mathcal{H}$ to decide their selection. Finally, lines 15-18 ensure to always select the highest valued human.

\vspace{-1.25mm}
\section{Experiments}\label{section:experiments}
We performed various experiments to demonstrate the effectiveness of PLACO. First, we show how our human label estimation method significantly outperforms na\"ive estimation methods described in Section \ref{subsection:naivemethods}. Next, we validate PLACO in terms of accuracy and cost spent on multi-way image classification tasks using the CIFAR-10H and ImageNet-16H datasets.

\subsection{Human Labels and AI Model}
Based on the annotations in the dataset we estimate true labels for humans for each configuration. The hard labels are determined by one's accuracy. Let there be $n$ human annotations available for an image \textit{x} in the dataset given by \textit{t(x)} = $\{t_1(x),t_2(x),....,t_n(x)\}$ and the ground truth is \textit{y(x)}. Given $\mathcal{K}$ classes, we define a distribution over all classes:
\begin{equation*}
    p(k|x) = 
    \begin{cases}
        0 & \text{if } k = y(x), \\
        \frac{g_x(k)}{\sum_{i \in \mathcal{Y'}(x)}g_x(i)} & \text{if } k \neq y(x) \text{ and } g_x(y(x)) \neq 1, \\
        \frac{1}{\mathcal{K}-1} & \text{if } k \neq y(x) \text{ and } g_x(y(x)) = 1
    \end{cases}
\end{equation*}
where $g_x(k)$ denotes the fraction of humans who predicted label \textit{k} for image \textit{x} and $\mathcal{Y'}(x) = \{1,2,3...,\mathcal{K}\} - \textit{y(x)}$. We assign hard labels $t^{a}(x)$ for a human with accuracy \textit{a} as follows:

\begin{equation*}
    t^{a}(x) = 
    \begin{cases}
        y(x) & \text{with probability } a, \\
        t \sim p(k|x) & \text{with probability } 1 - a
    \end{cases}
\end{equation*}

For both datasets CIFAR-10H and ImageNet-16H, we use this estimation method to generate true human labels.

To show the effectiveness of  the Human-AI team approach, we use a CNN model \cite{singh2023subset} that has 56.74\% accuracy for CIFAR-10H dataset and another CNN model \cite{Steyvers_Tejeda_2023imagenet} for ImageNet-16H dataset that has 43.40\% accuracy for the 10-way and 16-way classification tasks respectively. 

\subsection{Na\"ive Human Label Estimation}\label{subsection:naivemethods}
To perform human label estimation we assume only to have the model's probability vector $m(x)$ and the estimated human confusion matrices $\phi^{[i]}$. There are several ways to utilize the given information to generate an estimated label for each human. Here we describe some of the na\"ive methods for label estimation:
\begin{itemize}
    \item \textbf{Max-Max}: Consider the most probable label of each human for each instance. This can be deduced from that human's confusion matrix; the cell with the highest probability value will determine the estimated label. Note that this method generates the same label for a human irrespective of the current instance.
\begin{equation*}
    h_i(x) = \arg\max_{y_1 \in \mathcal{Y}} (\max_{y_2 \in \mathcal{Y}} {\phi^{[i]}_{y_1y_2}})
\end{equation*}
    \item \textbf{Random}: Assign any one label out of the $K$ labels as the estimated label to each humans such that all $K$ labels are equally likely.
    \item \textbf{Top-K}: Consider top $k$ most probable labels from the model probability vector $m(x)$. Choose any label $y_1$ at random and combine this information with the confusion matrix. Specifically, the estimated label of $i^{th}$ human is
\begin{equation*}
    h_i(x) = \arg\max_{y_2 \in \mathcal{Y}} {\phi^{[i]}_{y_1y_2}}
\end{equation*}
    
\end{itemize}

\subsection{Subset Selection Algorithms}\label{subsection:subset-selection}
We use three subset selection algorithms to show the effectiveness of our proposed approach PLACO. Cost constraints are added from one algorithm to the next. While Pseudo LB operates with the knowledge of true human labels, PLACO Greedy performs cost-effective subset selection using estimated human labels. PLACO LP operates in the scenario when a budget constraint is also introduced. In addition, we set the values of $V_{max} = 10^9$ and $\epsilon = 10^{-9}$ in our value function used in our experiments.

\begin{figure*}[t]
    \centering
    \includegraphics[width=0.245\linewidth]{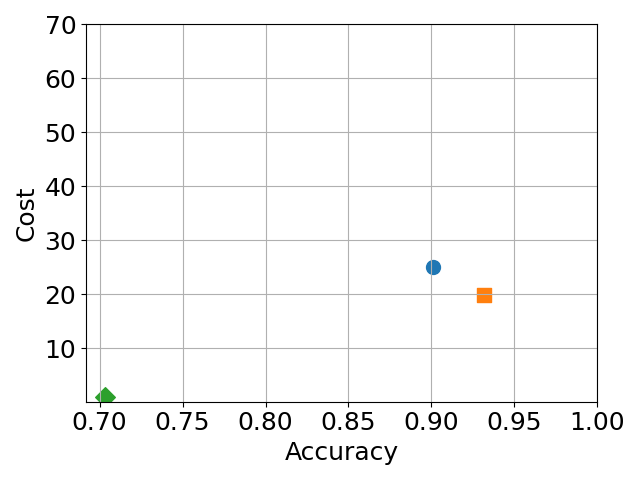}
    \includegraphics[width=0.245\linewidth]{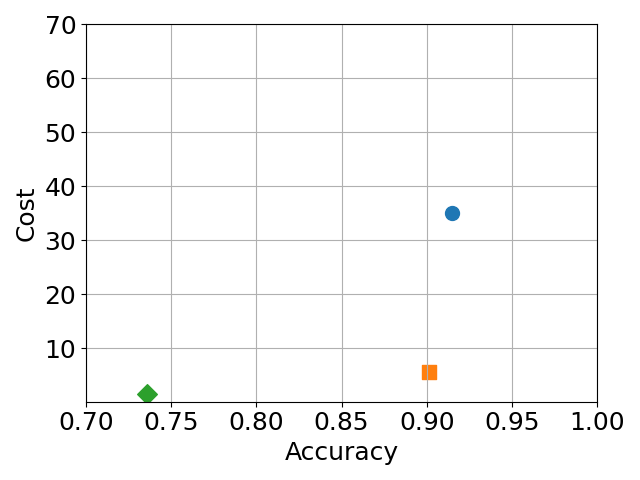}
    \includegraphics[width=0.245\linewidth]{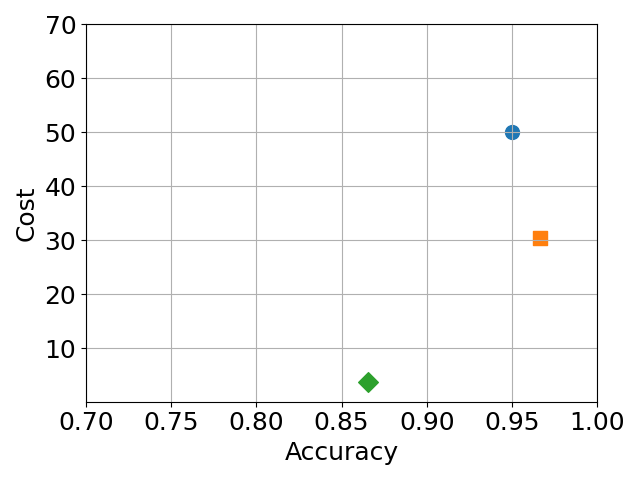}
    \includegraphics[width=0.245\linewidth]{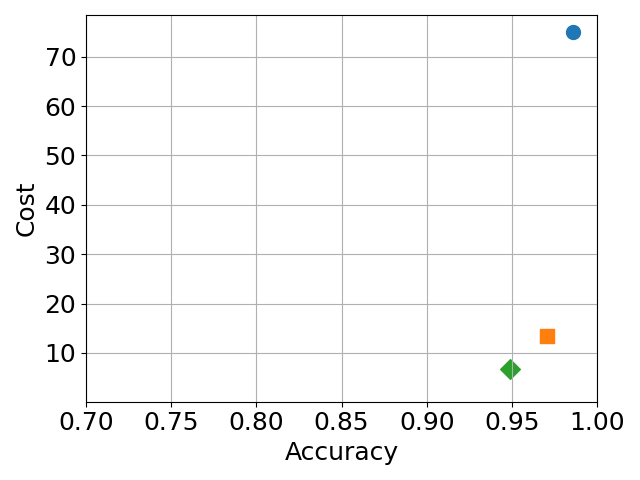}
    \includegraphics[width=0.245\linewidth]{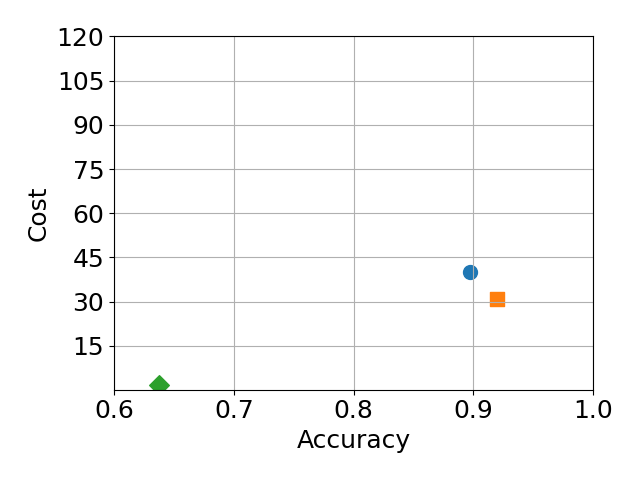}
    \includegraphics[width=0.245\linewidth]{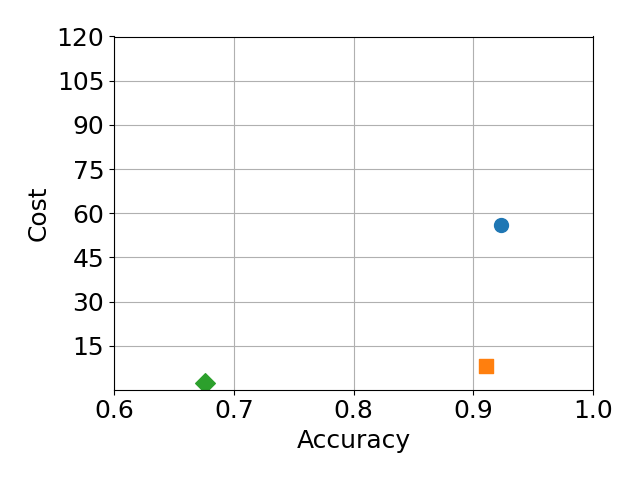}
    \includegraphics[width=0.245\linewidth]{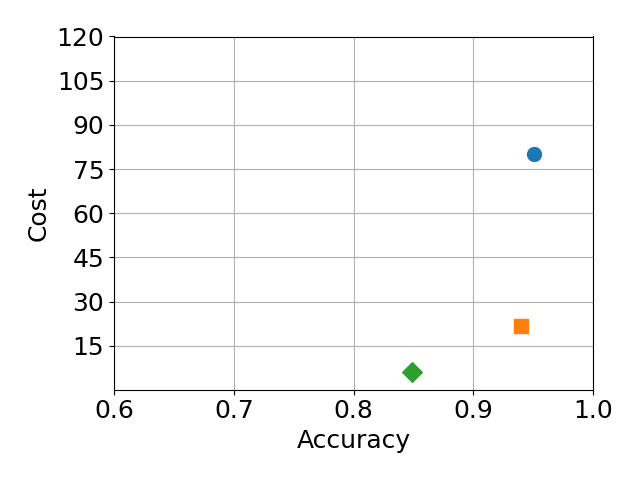}
    \includegraphics[width=0.245\linewidth]{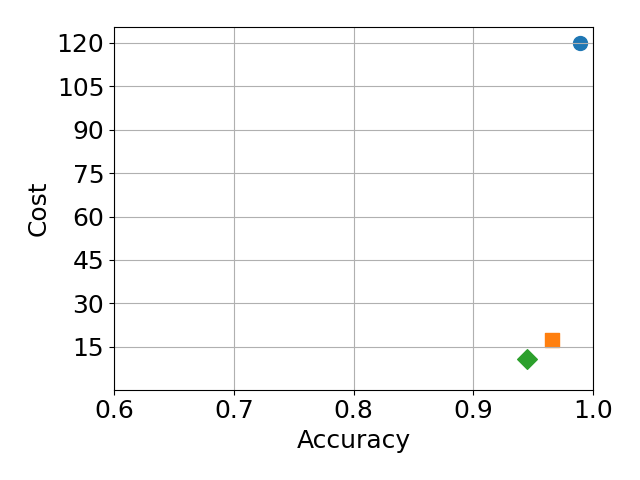}
    \includegraphics[width=0.4\linewidth, height=0.03\linewidth]{Figures/legend.png}
    \vspace{-3mm}
    \caption{Accuracy vs Cost Trade-Off Scatter Plot for  with different subset selection methods average over 10 runs. Each plot corresponds to a different human configuration of varying accuracies. Specifically, the number of humans in the configurations are 5, 7, 10 and 15 (from left to right in each row), with accuracies ranging from 0.3 to 0.9. The first row corresponds to CIFAR-10H (5000 training instances) and the second corresponds to ImageNet-16H (2500 training instances). A CNN model with accuracy 56\% is used as AI model for probabilistic output for CIFAR-10H and a another CNN model with accuracy 43\% is used for ImageNet-16H to combine with human labels.}
    \label{fig:tradeoffcifar}
\end{figure*}

\subsubsection{Pseudo LB: Pseudo Optimal Subset Selection}
Singh et al. \cite{singh2023subset} introduced this greedy approach to maximize the lower bound on the overall accuracy of the Human-AI team. The usage of this algorithm depicts the state-of-the-art performance when true human labels are known for all humans across all instances. Equation \ref{eq-pseudo-optimal} shows how Pseudo LB performs subset selection. 
\begin{equation}\label{eq-pseudo-optimal}
    S_x = \arg\max_{S} \max_{1 \leq j \leq \mathcal{K}} \left(\prod_{i\in S} \frac{\phi^{[i]}_{t_{i}(x)j}}{1 - \phi^{[i]}_{t_{i}(x)j}}\right)
\end{equation}
\subsubsection{PLACO Greedy: Cost-Effective Subset Selection}
Inspired from the above greedy approach, we introduce a cost-effective version. It maximizes a further lower bound which arises from human label estimation, which has been derived in Section \ref{subsection:value-function}. This approach depicts the upper bound on performance when only estimated human labels are known at the time of subset selection.
\subsubsection{PLACO LP: Cost-Effective LP Solver}
We formulate our subset selection problem as an integer linear programming problem as
\begin{align*}
    \max_S \sum_{i \in S} \log V_i(x, y^*) \cdot e_i
\end{align*}
\vspace{-3mm}
subject to
\vspace{-3mm}
\begin{align*}
    \sum_{i \in S} c_i \cdot e_i \leq B \; , \; e_i \in \{0, 1\}
\end{align*}
where, $B$ is given budget, $V_i(x, y^*)$ is the value of human $i$ for an instance $x$, $e_i$ is the decision variable whether or not to choose the $i^{th}$ human, and $c_i$ is the cost of the $i^{th}$ human.

\subsection{Cost Function and Budget}
In order to free ourselves of any assumptions on the cost distribution of the human set $\mathcal{H}$, we choose an arbitrary cost function. Specifically, for each run, we fix the cost of the $i^{th}$ human on a given $k$-way classification task as a random real number $c_i$ where $c_i \in (0, k)$. The budget $B$ used in PLACO LP is derived as a fraction of the upper bound cost of the entire human set $\mathcal{H}$. For our experiments, we use the budget $B = 0.05 * |\mathcal{H}| * k$. Note that this cost function is a worst case scenario. A well-behaved cost function such as a function of human accuracy will only improve our framework's performance.

\vspace{-1.25mm}
\section{Results and Inferences}\label{section:results}

Figure \ref{fig:estimation_match} shows the estimation match, that is, the average fraction of correctly estimated human labels on a given instance. It is clear from the figure that our posterior probability method estimates human labels with significantly higher accuracy than na\"ive methods such as Max-Max, Top K, and Random. It can be noted that the selected subsets are vastly effective in spite of the estimation match being in the range 0.25 to 0.35. This is expected as this range nearly matches the fraction of ideal humans in a given configuration. The decrease in estimation match from CIFAR-10H to ImageNet-16H is due to more classes and lesser instances to construct human confusion matrices in ImageNet-16H as compared to CIFAR-10H.

Figure \ref{fig:accuracycifar} depicts a performance evaluation of various subset selection methods, namely Pseudo LB, PLACO Greedy, and PLACO LP, as elaborated in Section \ref{subsection:subset-selection}, on the CIFAR-10H and ImageNet-16H datasets. Pseudo LB depicts the best possible accuracy given true human labels of all humans in all instances. On the other hand, PLACO Greedy exhibits the best possible accuracy using the PLACO framework, i.e., using estimated human labels. As a direct result of Lemma \ref{ratio-lower-bound}, we expect PLACO Greedy to perform as well as Pseudo LB in the best case when maximizing the lower bound in equation \ref{lowerbound_ideal} allows us to select the same subset as Pseudo LB. Note that the same subset can be selected using PLACO Greedy at a significantly lower cost when compared to Pseudo LB. Our experimental results in Figure \ref{fig:accuracycifar} align perfectly with these expectations. If further cost optimisation is required, a budget constraint can be introduced, and the subset selection task can be formulated as a linear optimization problem. It is expected that such a subset selection will keep the cost strictly bounded;  it can possibly offer lower performance due to additional constraints. PLACO LP demonstrates this tradeoff as expected.

From Figure \ref{fig:tradeoffcifar} , it is evident that Pseudo LB exhibits high accuracy along with high cost, as it operates without any cost constraints and relies on prior knowledge of true human labels, thus considering the cost of all individuals. PLACO Greedy achieves accuracy at par with Pseudo LB with a significantly lower cost. This is because PLACO uses estimated human labels for subset selection instead of true human labels. Hence, it only takes into account the cost of the true human labels of the selected subset of humans. PLACO LP does further cost optimisation by introducing a budget constraint on the subset to be selected. While the cost is bounded by the given budget, we see that accuracy is negatively affected due to the insufficient budget. However, as the total number of humans increases, the same budget allows more ideal humans to be selected which results in greater accuracy, as the budget is proportional to $|\mathcal{H}|$.

\vspace{-1.25mm}
\section{Discussions and Conclusions}
In this paper, we propose a framework, PLACO, to select cost-optimized and highly performant human subsets to form effective Human-AI teams in a practical setting where we do not have prior knowledge of human true labels. We have empirically shown its efficacy on accuracy and cost parameters on two image datasets, CIFAR-10H and ImageNet-16H. 
    
An interesting avenue of exploration would be the setting in which human cost is correlated with their respective domain knowledge, as currently we only consider a random cost function. In addition, although we predict that PLACO can perform just as well on tasks other than image classification, one can empirically verify this on relevant datasets. Also, fairness in AI-decision making is becoming increasingly relevant everyday. Hence, studying the fairness impacts of the PLACO framework can also prove very useful.

\vspace{-1.25mm}
\section{Acknowledgements}
This work has been partially supported by the National Science and Technology Council, Taiwan NSC-112-2927-I-194-001 and NSC-111-2927-I-194-001 along with SERB grant CRG/2023/002142.

\bibliography{/m1197}

\end{document}